\documentclass[aos,preprint,12pt]{imsart}
\usepackage{amsmath, amssymb, amsthm}
\usepackage{mathrsfs}
\usepackage{fullpage}
\usepackage{graphicx, color, hyperref, epstopdf}
\usepackage[usenames,dvipsnames]{xcolor}
\usepackage{comment}

\newcommand{\boldX}{\mathbf{X}}

\newtheorem{theorem}{Theorem}
\newtheorem{lemma}[theorem]{Lemma}

\let \hat \widehat

\newcommand{\vol}{\operatorname{vol}}

\newcommand*{\defeq}{\mathrel{\vcenter{\baselineskip0.5ex \lineskiplimit0pt
                     \hbox{\scriptsize.}\hbox{\scriptsize.}}}%
                     =}

\begin{document}

\title{Tight lower bounds for homology inference}
\runtitle{Tight lower bounds for homology inference}

\begin{aug}
\author{\fnms{Sivaraman} \snm{Balakrishnan}\ead[label=e1]{sbalakri@cs.cmu.edu}},
\author{\fnms{Alessandro} \snm{Rinaldo}\ead[label=e4]{arinaldo@cmu.edu}},
\author{\fnms{Aarti} \snm{Singh}\ead[label=e5]{aarti@cs.cmu.edu}}
\and
\author{\fnms{Larry} \snm{Wasserman}\corref{}
\ead[label=e6]{larry@cmu.edu}}

\runauthor{Balakrishnan et al}

\affiliation{
    School of Computer Science and Statistics Department\\
    Carnegie Mellon University
}

\address{
    School of Computer Science\\
    Carnegie Mellon University\\
    Pittsburgh, PA 15213\\
    \printead{e1}\\
    \printead{e5}
}

\address{
    Department of Statistics\\
    Carnegie Mellon University\\
    Pittsburgh, PA 15213\\
    \printead{e4}\\
    \printead{e6}
}
\end{aug}
\begin{quote}
The homology groups of a
manifold are important topological invariants
that provide an algebraic summary of the
manifold. These groups contain rich topological information, for instance, about the
connected components, holes, tunnels and
sometimes the dimension of the manifold.
In earlier work \cite{Balakrishnan2011}, we have
considered the statistical problem of estimating the homology 
of a manifold from noiseless samples and from
noisy samples under several different 
noise models. We derived upper and lower
bounds on the minimax risk for this problem.
In this note we revisit the noiseless case. In \cite{Balakrishnan2011},
we used Le Cam's lemma to establish the lower bound 
\footnote{The asymptotic notation in both the upper and lower bounds
hide constants that could depend on the dimensions $d$ and $D$.}
$$R_n = \Omega\left( \exp \left( -n \tau^d\right)\right)$$
for $d \geq 1$ and $D > d$. 
In the noiseless case the upper bound follows from the work of 
\cite{niyogi2008}, who show that 
$$R_n = O \left(\frac{1}{\tau^d} \exp \left( -n\tau^d\right)\right).$$
In this note we use a different construction based on the 
direct analysis of the likelihood ratio test to show that
$$R_n = \Omega \left(\frac{1}{\tau^d} \exp \left( -n\tau^d\right)\right),$$
as $n \rightarrow \infty$
thus establishing rate optimal asymptotic minimax bounds for the problem. The techniques we use
here extend in a straightforward way to the noisy settings considered
in \cite{Balakrishnan2011}. Although, we do not consider the extension here
non-asymptotic bounds are also straightforward.
\end{quote}

\section{Introduction}
Let $M$ be a 
$d$-dimensional manifold embedded in $\mathbb{R}^D$
where $d \leq D$.
The {\em homology groups}  ${\cal H}(M)$  of $M$ (see \cite{hatcher01}), are an algebraic summary of the properties of $M$. The homology groups of a manifold describe its topological features
such as its connected components, holes, tunnels, etc.


In this note we study the problem of 
estimating the homology of a manifold $M$ from a sample
$\boldX = \{X_1,\ldots, X_n\}$.
Specifically, we bound the minimax risk
\begin{equation}
\label{eqn::risk}
R_n\equiv \inf_{\hat{\cal H}}\sup_{Q\in {\cal Q}}
Q^n \Bigl(\hat{\cal H} \neq {\cal H}(M)\Bigr)
\end{equation}
where
the infimum is over \emph{all} estimators
$\hat{\cal H}$ of the homology of $M$ and the supremum is over
appropriately defined classes of distributions ${\cal Q}$ for $Y$.
Note that $0 \leq R_n \leq 1$ with $R_n=1$
meaning that the problem is hopeless.
Bounding the minimax risk is equivalent to
bounding the {\em sample complexity} of the best possible estimator, 
defined by
$n(\epsilon) = \min\bigl\{n:\ R_n \leq \epsilon\bigr\}$
where $0 < \epsilon < 1$.

We assume that the sample $\boldX \subset \mathbb{R}^D$ constitutes a set of observations of an unknown $d$-dimensional manifold $M$, with $d < D$, whose homology  we seek to estimate. The distribution of the sample depends on the properties of the manifold $M$ as well as on the distribution of points on $M$. We consider the collection 
$$\mathcal{P} \equiv \mathcal{P}(\mathcal{M}) \equiv \mathcal{P}(\mathcal{M},a)$$ of all probability distributions supported over manifolds $M$ in ${\cal M}$ having densities $p$ with respect to the volume form on $M$ uniformly bounded from below by a constant $a>0$, i.e. 
$0 < a \leq p(x) < \infty $ for all $x\in M$. 




{\bf Manifold Assumptions.}
We assume that the unknown manifold $M$ is a $d$-dimensional smooth compact Riemannian 
manifold without boundary embedded in
the compact set ${\cal X} = [0,1]^D$. We further assume that the volume of the manifold 
is bounded from above by a constant which can depend on the dimensions $d, D$, 
i.e. we assume $\vol(M) \leq C_{D,d}$.
We will also make the further assumption that $D > d$.
The main regularity condition we impose on $M$ is that its {\em condition number} be not too large. 
The {\em condition number} $\kappa(M)$ (see \cite{niyogi2008}) is $1/\tau$, where
$\tau$ is the largest number
such that the open normal bundle about $M$ of radius $r$ 
is imbedded in $\mathbb{R}^D$ for every $r < \tau$.
For $\tau >0$ let
$${\cal M} \equiv {\cal M}(\tau)=\Bigl\{M: \kappa(M) \geq \tau \Bigr\}$$
denote the set of all such manifolds with condition number no smaller than $\tau$.
A manifold with small condition number 
does not come too
close to being self-intersecting.
%

\subsection{Lower bounding the minimax risk}
In this note we will lower bound the minimax risk by considering a related \emph{testing}
problem. 

Before describing the hypotheses we describe the null and alternate manifolds.
The null manifold $M_0$ is a collection of $m$, $d$-spheres of radius $\tau$,
denoted $S_1,\ldots,S_m$, with
centers on one face of the unit
hypercube in $d + 1$ dimensions ($M_0$ is embedded in a space of dimension $D$ which is
of dimension at least $d + 1$), with spacing between adjacent centers $= 4\tau$.
It is easy to see that $$m = O \left( \frac{1}{(4 \tau)^d} \right)$$
because the manifold must be completely in 
$[0,1]^D$, and that the manifold has condition number
at least $1/\tau$. We will use $$m = \Theta \left( \frac{1}{(4 \tau)^d} \right)$$
in this note.
Let $P_0$ denote the uniform distribution on $M_0$.

The alternate manifolds are a collection $\{ M_{1i}: i \in \{1,\ldots,m\} \}$,
where $M_{1i}$ is $M_0$ with $S_i$ removed. Let $\pi$ denote
the uniform distribution on $\{1,\ldots,m\}$, and $P_{1i}$ denote
the uniform distribution on $M_{1i}$.

We need to ensure that the density $p$ is lower bounded by a constant.
Note that 
the total $d$-dimensional volume of $M_0$ is $v_d \tau^d m$, and so
$$p(x) \geq \frac{1}{v_d \tau^d m}$$
where $v_d$ is the volume of the $d$-dimensional unit ball. This is 
$\Omega(1)$ as desired. A similar argument works for $M_{1i}$.


Consider the following testing problem:
\begin{eqnarray*}
& H_0 &: \boldX \sim P_0 \\
& H_1 &: \boldX \sim P_{1i}~\mathrm{with}~i \sim \pi 
\end{eqnarray*}
A \emph{test} $T$,  is a measurable function of $\boldX$, in particular
$T: \boldX \rightarrow \{0,1\}$, and its risk is defined as
$$R_n^T \defeq \mathbb{P}_{H_0}(T(\boldX) = 1) + \mathbb{P}_{H_1}(T(\boldX) = 0)$$

The relationship between testing and estimation is standard \cite{lehmann05}. In our case it is easy to see
that the estimation minimax risk of Equation \ref{eqn::risk} satisfies,
$$R_n^T \leq 2R_n$$
and so it suffices to lower bound $R_n^T$ to obtain a lower bound on $R_n$. This relation is
a straightforward consequence of the fact that $\mathcal{H}(M_0) \neq \mathcal{H}(M_{1i})$ for
every $i$ (since they have different number of connected components), and so any estimator can be used in the testing problem described.

The optimal test for the hypothesis testing problem described is the likelihood ratio test,
$$T(\boldX) = 0~\mathrm{if~and~only~if}~L(\boldX) \leq 1$$
where $$L(\boldX) = \frac{L_1(\boldX)}{L_0(\boldX)}$$ where $L_1(\boldX)$ and $L_0(\boldX)$
are likelihoods of the data under the alternate and null respectively. 

\subsection{Coupon collector lower bound}
We begin with a theorem from \cite{motwani95}.
\begin{lemma} [Theorem 3.8 of \cite{motwani95}]
\label{lem::ccasymp}
Let the random variable $X$ denote the number of trials for
collecting each of the $n$ types of coupons. Then for any constant $c \in \mathbb{R}$,
and $m = n \log n - cn$,
$$\lim_{n \rightarrow \infty} \mathbb{P}(X > m) = 1 - \exp\left( - \exp\left(c\right)\right)$$
\end{lemma}

\section{Main result}
\begin{theorem}
For any constant $\delta < 1$, we have
$$R_n \geq  \Omega\left( \min\left(\frac{1}{\tau^d} \exp \left( -n\tau^d\right) , \delta \right)\right)$$
as $n \rightarrow \infty$.
\end{theorem}
\begin{proof}
Notice that since $$m = \Theta \left( \frac{1}{(4 \tau)^d} \right)$$ the theorem is 
implied by the statement that
$$n = m \log m + m \log \left( \frac{1}{\delta} \right)  \implies R_n \geq c \delta$$
for some constant $c$.
We will focus on proving this claim.

Let us consider the case when samples are drawn according to $P_0$. From Lemma \ref{lem::ccasymp} we have that if $$n = m \log m + m  \log \left( \frac{1}{\delta} \right)$$
then the probability with which we do not see a point in each of the $m$ spheres is 
$$1 - \exp (-\exp (- \log 1/\delta)) \geq c \delta$$
since $\delta < 1$, for some constant $c$.
It is easy to see that if we do not see a point in each of the $m$ spheres then
$$L(\boldX) \geq \frac{1}{m} \frac{1}{(1 - 1/m)^n} \defeq T_{m,n}$$

When $n = m \log m + m \log \left(\frac{1}{\delta} \right)$,
$$T_{m,n} \rightarrow \frac{1}{\delta} > 1$$
so asymptotically the likelihood ratio test always rejects the null.

From this we can see the probability of a Type I error $\rightarrow c \delta$, and 
$R_n^T \geq c \delta$, which gives 
$$R_n \geq \frac{c}{2} \delta$$ 
as desired.
\end{proof}

\section{Discussion}
In this note we have established tight minimax rates for the problem of homology inference in 
the noiseless case. The intuition behind the construction extends to the noisy cases considered in
\cite{Balakrishnan2011} in a straightforward way. 

Although the bound we have shown is an asymptotic lower bound, a finite sample lower bound follows
in a straightforward way by replacing the asymptotic calculation in Lemma \ref{lem::ccasymp} with 
finite sample estimates.

We also expect similar constructions to be useful 
in establishing tight lower bounds 
for the problems of manifold estimation in Hausdorff distance considered in \cite{Genovese2012,Genovese.JMLR},
and for the problem of estimation of persistence diagrams in bottleneck distance considered in \cite{chazal2013b}.

{
{
\bibliographystyle{unsrt}
\bibliography{biblio_small}
}
}
%
%
%

\end{document}